\documentclass[sn-mathphys-num]{sn-jnl}% Math and Physical Sciences Numbered Reference Style 
\usepackage{graphicx}%
\usepackage{multirow}%
\usepackage{amsmath,amssymb,amsfonts}%
\usepackage{amsthm}%
\usepackage{mathrsfs}%
\usepackage[title]{appendix}%
\usepackage{textcomp}%
\usepackage{manyfoot}%
\usepackage{booktabs}%
\usepackage{algorithm}%
\usepackage{algorithmic}
\usepackage{listings}%
\usepackage{mathtools}
\usepackage{amsthm}
\usepackage{makecell}
\usepackage[table]{xcolor}

\theoremstyle{thmstyleone}%
\newtheorem{theorem}{Theorem}%  meant for continuous numbers
\theoremstyle{thmstyletwo}%

\theoremstyle{thmstylethree}%

\begin{document}

\title[Article Title]{Cloud Model Characteristic Function Auto-Encoder: Integrating Cloud Model Theory with MMD Regularization for Enhanced Generative Modeling}


\author[1,2]{\fnm{Biao} \sur{Hu}}

\author*[1]{\fnm{Guoyin} \sur{Wang}}

\affil[1]{\orgdiv{School of Computer Science and Technology}, \orgname{Chongqing University of Posts and Telecommunications}, \orgaddress{\city{Chongqing}, \postcode{400065}, \country{China}}}

\affil[2]{\orgdiv{School of Big Data and Internet of Things}, \orgname{Chongqing Vocational Institute of Engineering}, \orgaddress{\city{Chongqing}, \postcode{402260}, \country{China}}}


\abstract{We introduce Cloud Model Characteristic Function Auto-Encoder (CMCFAE), a novel generative model that integrates the cloud model into the Wasserstein Auto-Encoder (WAE) framework. By leveraging the characteristic functions of the cloud model to regularize the latent space, our approach enables more accurate modeling of complex data distributions. Unlike conventional methods that rely on a standard Gaussian prior and traditional divergence measures, our method employs a cloud model prior, providing a more flexible and realistic representation of the latent space, thus mitigating the homogenization observed in reconstructed samples. We derive the characteristic function of the cloud model and propose a corresponding regularizer within the WAE framework. Extensive quantitative and qualitative evaluations on MNIST, FashionMNIST, CIFAR-10, and CelebA demonstrate that CMCFAE outperforms existing models in terms of reconstruction quality, latent space structuring, and sample diversity. This work not only establishes a novel integration of cloud model theory with MMD-based regularization but also offers a promising new perspective for enhancing autoencoder-based generative models.}

\keywords{Cloud Model, Auto-Encoder, Generative Model, MMD Regularization}



\maketitle

\section{Introduction}\label{sec:intro}

\begin{figure*}[!htb]
\begin{center}
\centerline{\includegraphics[width=\textwidth]{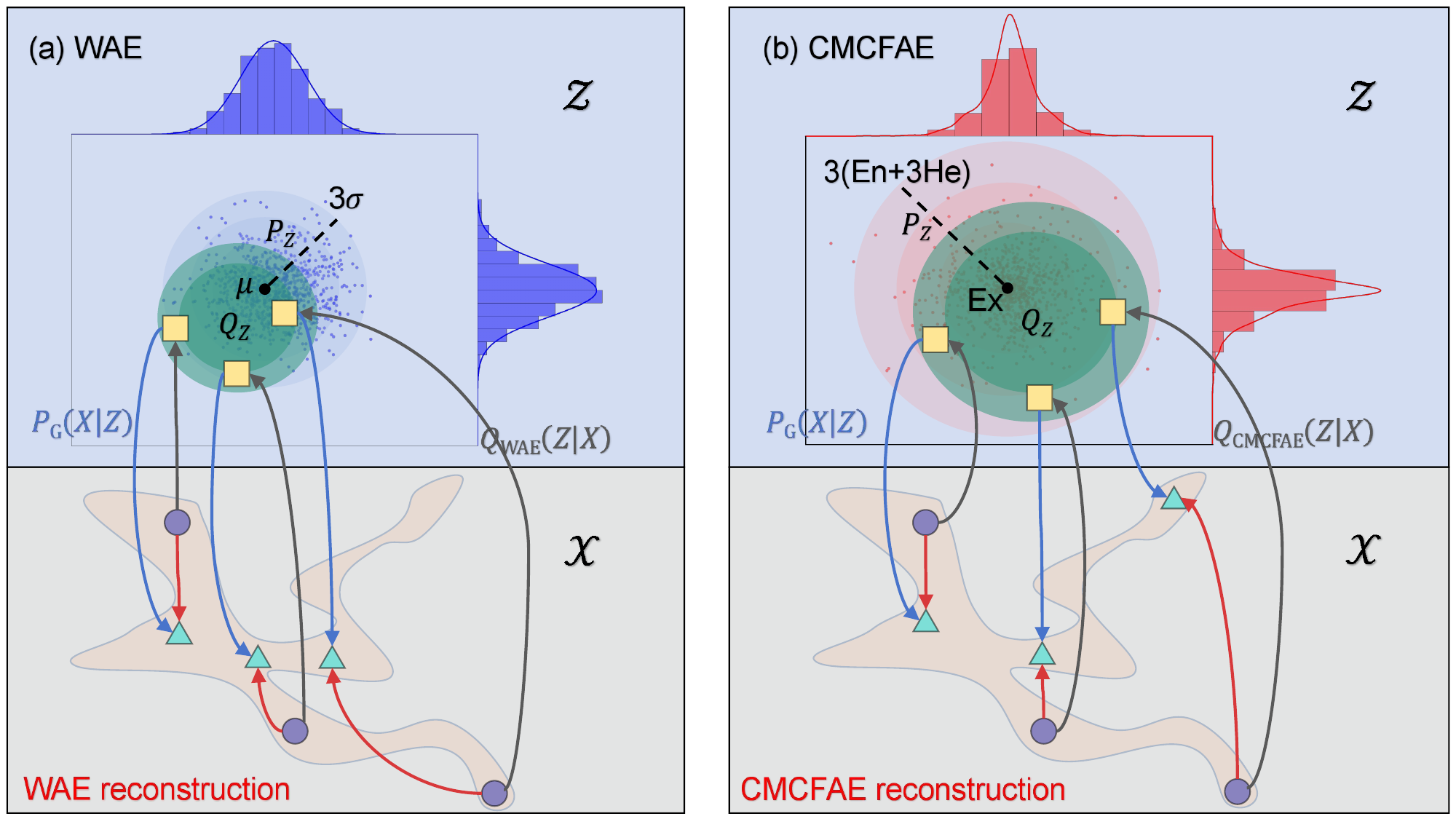}}
\caption{Comparison of the General WAE Framework and CMCFAE. CMCFAE is formulated within the general WAE framework, adhering to its standard configuration, where the optimization objective comprises two components: the reconstruction cost and the regularization term. The primary distinction between CMCFAE and WAE lies in how the regularization term is computed. As illustrated in Figure (a), WAE adopts a standard Gaussian distribution as the prior. The regularization term ensures that the aggregated posterior $Q_Z := \int Q(Z|X) dP_X$ aligns with the prior $P_Z$. This prior-induced constraint may result in sample homogenization during the reconstruction process. As depicted in Figure (a), the original samples (purple spheres) exhibit substantial diversity, whereas the reconstructed samples (green triangles) show diminished variation, leading to homogenized reconstruction results. As illustrated in Figure (b), CMCFAE employs the cloud model as the prior. Leveraging the cloud model’s more flexible sampling space (adjustable via $He$), CMCFAE achieves reconstructions of comparable quality even when the original samples exhibit considerable diversity.}
\label{fig:CMCFAE}
\end{center}
\end{figure*}

Generative models have made significant progress in learning complex, high-dimensional distributions. They are capable of simulating high-dimensional, intricate datasets, with the goal of generating samples that closely resemble the input data \cite{kingma2013auto, radford2015unsupervised, van2016pixel, arjovsky2017wasserstein, wang2023slot}. Variational Autoencoders (VAEs) model high-dimensional data probabilistically and are founded on elegant theoretical principles. In VAEs, the Kullback-Leibler (KL) divergence is employed in the latent space to quantify the distance between the latent variable distribution and the target distribution; this divergence is minimized via variational inference techniques. Since their introduction \cite{kingma2013auto}, VAEs have become a popular framework for generative modeling.

Numerous improvements to VAEs have been proposed, primarily focusing on image generation tasks. However, challenges remain concerning the quality of generated samples, which are frequently blurry. One potential reason for the discrepancy between generated and real samples is the overly simplistic prior distribution \cite{nalisnick2016stick, dai2023optimized} or posterior \cite{rezende2015variational}, along with the excessive regularization imposed by the KL divergence used to compare the latent variable distribution with the target distribution \cite{higgins2017beta}. The prior plays a critical role in VAEs, as it largely determines model performance \cite{johnson2016composing, hoffman2016elbo}. In VAEs, a simple prior—typically a Gaussian prior—is commonly adopted. Some approaches employ Gaussian Mixture Models (GMMs) as the prior to enhance model performance \cite{dilokthanakul2016deep, tomczak2016improving}. However, these methods predominantly rely on Monte Carlo simulations, which can adversely impact training stability when the sample size is limited.

The aforementioned efforts represent improvements within the framework of VAEs. In contrast, the development of the WAE-MMD models \cite{tolstikhin2017wasserstein} introduces the Wasserstein metric, which relaxes the constraints imposed by variational methods and facilitates the generation of higher-quality images. \cite{2018Sliced} introduced the Sliced-Wasserstein Auto-Encoder (SWAE), which incorporates the Sliced-Wasserstein distance, thereby significantly accelerating its computation. \cite{knop2020cramer} introduced the Cramer-Wold distance between distributions, which is derived from the MMD distance and a novel Cramer-Wold kernel, and features a cost function with a closed-form analytical expression. However, the Cramer-Wold kernel is limited to measuring the discrepancy between a sample and a mixture of radial Gaussian distributions, which imposes certain constraints. \cite{bruck2024generative} was the first to describe the Maximum-Mean-Discrepancy (MMD) metric from the perspective of characteristic functions, directly incorporating the characteristic function into the model's loss function to provide a more detailed characterization of distributional differences. Furthermore, owing to the generality of characteristic functions, this approach mitigates the limitations inherent in mixtures of radial Gaussian distributions.

The main contributions of this paper include integrating the cloud model into the WAE framework, deriving its characteristic function, and proposing a regularizer based on the cloud model's characteristic function. The cloud model is a probabilistic model renowned for its robust data representation capabilities. When employed as a prior in VAEs, it can expand the latent space, thereby enhancing the likelihood of capturing a broader range of features during the sampling process \cite{dai2023optimized, liu2023cloud}.

\section{Related Work}\label{sec:Related Work}
Auto-Encoders (AEs), particularly Variational Auto-Encoders (VAEs), have been extensively studied for their effectiveness in learning latent representations of data \citep{kingma2013auto, higgins2017beta}. Among these, Wasserstein Auto-Encoders (WAEs) \citep{tolstikhin2017wasserstein} provide an alternative framework to VAEs, mitigating challenges associated with KL divergence and enhancing reconstruction quality through optimal transport theory.

\textbf{Wasserstein Auto-Encoders (WAEs)}. WAEs are a family of generative models in which the autoencoder utilizes stochastic gradient descent (SGD) to estimate and minimize the Wasserstein metric between the generative model $P_\theta(X)$ and the data distribution $P_{data}(X)$. Subsequent research has extended WAEs by integrating various divergence measures and distance metrics \citep{tolstikhin2017wasserstein, 2018Sliced, knop2020cramer, nakagawa2022gromov}. According to the theoretical analysis in \citep{bousquet2017optimal}, this family of generative models is formulated as a representation learning approach from the perspective of optimal transport (OT). The optimization objective of WAEs is equivalent to that of InfoVAE \citep{zhao2019infovae}, which learns variational autoencoder models by maximizing the mutual information of the probabilistic encoder.

\textbf{Maximum Mean Discrepancy (MMD)}. The Maximum Mean Discrepancy (MMD) \citep{gretton2006kernel} has become a widely adopted metric for measuring the divergence between probability distributions in machine learning. Unlike traditional measures such as the KL divergence, MMD leverages kernel functions to provide a non-parametric and flexible approach, making it particularly suitable for generative frameworks such as WAEs \citep{tolstikhin2017wasserstein} and GANs \citep{binkowski2018demystifying}. Additionally, MMD has been employed to regularize the training of diffusion models \citep{li2023error} and to fine-tune them for accelerated sampling \citep{aiello2024fast}.  

\textbf{Cloud Model (CM)}. The Cloud Model (CM) \citep{wang2014generic} is a mathematical framework extensively used in uncertainty representation and knowledge discovery. Previous research has predominantly focused on its qualitative properties and applications in fields such as data classification and uncertainty analysis \citep{wang2016cloud, xie2021novel, liu2023large}. However, the theoretical foundations of the CM face notable limitations. Specifically, its probability density function (PDF) lacks an analytical solution \citep{li2009new}, which impedes its precise mathematical characterization and limits its broader application in stochastic modeling. The absence of an analytical expression for the CM's PDF significantly constrains its use in generative modeling, where accurate probability representations are often essential. This challenge is particularly prominent in frameworks like WAEs, which require clear mathematical formulations for regularization terms such as MMD. As a result, integrating the CM into advanced generative models remains an unresolved issue. In this study, we tackle this longstanding challenge by deriving the characteristic function of the CM, offering an alternative mathematical representation that enables modeling its stochastic processes without relying on the intractable PDF. By utilizing these characteristic functions, we incorporate the CM into the WAE framework, demonstrating its ability to improve generative performance by capturing complex data distributions while preserving mathematical rigor.

\section{Methodology}\label{sec:Methodology}

In this section, we provide a concise overview of the Cloud Model (CM) theory, which is widely utilized for representing uncertainty and modeling stochastic processes. The Cloud Model integrates the strengths of fuzzy theory and probability theory, making it an effective tool for capturing uncertainty in diverse applications. However, the absence of an analytical probability density function (PDF) creates challenges when directly incorporating the CM into traditional models. To overcome this limitation, we derive the characteristic function for the CM, which offer a practical approach to representing the uncertainty inherent in the model. These characteristic function serve as an alternative representation of the stochastic processes governed by the CM and play a crucial role in integrating the CM into the WAE framework.

\subsection{Cloud Model}

\textbf{Cloud Model (CM)}, proposed by \citep{li2009new}, is a mathematical model that integrates fuzzy set theory and probability theory to represent uncertainty. As illustrated in Figure~\ref{CloudModel}, CM consists of three key components: 
\begin{itemize}
    \item \textbf{Expectation (Ex)}: Represents the central tendency or mean of the cloud model. 
    \item \textbf{Entropy (En)}: Quantifies the uncertainty, analogous to the standard deviation, indicating the spread of data. 
    \item \textbf{Hyper-Entropy (He)}: Refines the entropy by adjusting the distribution’s spread, offering a higher-order measure of uncertainty. 
\end{itemize} 
By adjusting the expectation and entropy, CM can model diverse types of uncertain distributions, making it particularly well-suited for handling uncertainty in generative models. Using these three parameters, CM samples can be generated through the Forward Cloud Generator (FCG) \citep{li2009new}. The specific generation algorithm is provided in Algorithm~\ref{alg:FCG}.  

\begin{figure}[!htb]
\begin{center}
\centerline{\includegraphics[width=8cm]{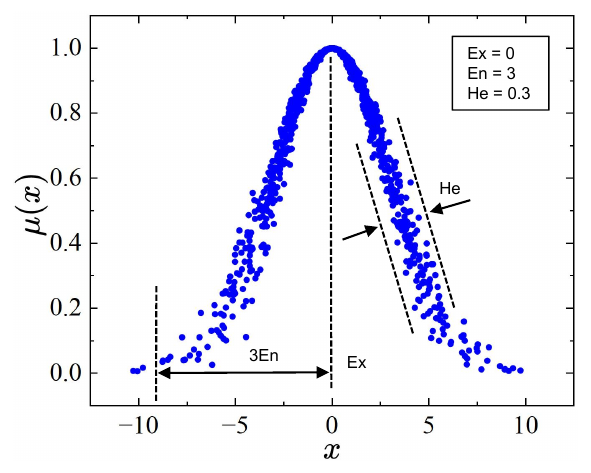}}
\caption{A schematic diagram of the cloud model is presented with parameters $Ex = 0$, $En = 3$, and $He = 0.3$. Here, $Ex$ represents the \textbf{mathematical expectation} of the random variable $X$, while $En$ quantifies its \textbf{uncertainty}, analogous to standard deviation. The parameter $He$ captures \textbf{entropy-based uncertainty}, which reflects the unevenness or dependencies within the random variable. Specifically, $He$ serves as a measure of deviation from a normal distribution, allowing the generalized normal distribution to better represent complex real-world data. It is important to note that this schematic does not depict the probability density function (PDF) of the random variable $X$. Instead, it illustrates the function $\mu(X)$, which highlights the practical significance of the three key numerical characteristics within the cloud model. }
\label{CloudModel}
\end{center}
\end{figure}

\begin{algorithm}[tb]
   \caption{Forward cloud generator: FCG(\(Ex\), $En$, \(En\), \(He\))}
   \label{alg:FCG}
\begin{algorithmic}
\STATE {\bfseries Input:} digital characteristics of CM: \(Ex\), \(En\), \(He\) and number of generated samples \(n\).
\STATE {\bfseries Output:} \(n\) samples \(x_i\) and their determinacy \(\mu(x_i)\)(\(i = 1,2,\dots,n\)).
\FOR{\(i=1\)  to \(n\)}
   \STATE Generate a normal random number \(s_i = R_N(En,He)\) with expectation \(En\) and variance \(He^{2}\).
   \STATE Generate a normal random number \(x_i = R_N(Ex, \vert s_i\vert)\) with expectation \(Ex\) and variance \(s_i^2\).
   \STATE Compute the certainty \(\mu(x_i) = \exp{\left(-\dfrac{(x_i-Ex)^2}{2s_i^2}\right)}\).
\ENDFOR
\STATE {\bfseries Return:} \(x_i\), \(\mu(x_i)\).
\end{algorithmic}
\end{algorithm}

Based on Algorithm~\ref{alg:FCG}, we can treat all the samples $x_i$ as realizations of a random variable $X$, and all the corresponding uncertainty values $En_i'$ as an intermediate random variable $S$, establishing a conditional probability relationship between them. Therefore, the probability density function (PDF) of the Cloud Model can be formulated as follows:  

First, the random variable $S$ follows a normal distribution with mean $En$ and variance $He^2$, given by:  

\begin{equation*}
\begin{aligned}
    f(s) = \dfrac{1}{\sqrt{2 \pi He^2}}\exp\left(-\dfrac{(s-En)^2}{2He^2}\right),
\end{aligned}
\end{equation*}
when $s = \sigma$, the random variable $X$ follows a Gaussian distribution with mean $Ex$ and variance $\sigma^2$. The conditional probability density function of $X$ is expressed as:  

\begin{equation*}
\begin{aligned}
    f(x|s=\sigma) = \dfrac{1}{\sqrt{2 \pi \sigma^2}}\exp\left(-\dfrac{(x-Ex)^2}{2\sigma^2}\right)
\end{aligned}
\end{equation*}

By applying the conditional probability density formula, the probability density function of the Cloud Model is derived as follows: 
\begin{equation} \label{eq:E1}
\begin{aligned}
    f(x) &= \int_{-\infty}^{+\infty} f(x|s=\sigma)f(\sigma){\rm d}\sigma
\end{aligned}
\end{equation}

Equation~\eqref{eq:E1} does not have a closed-form analytical solution, which prevents the Cloud Model from possessing a well-defined probability density function similar to traditional probability distributions. This lack of an explicit analytical PDF poses challenges for incorporating the Cloud Model into generative models, particularly in terms of optimization and data generation, as conventional PDF-based approaches cannot be directly applied.  

\subsection{Cloud Model Characteristic Function}

Due to the absence of an analytical probability density function (PDF) 
in the Cloud Model, conventional optimization methods, such as maximum likelihood estimation, cannot be directly employed for training. To overcome this limitation, we introduce the characteristic function of Cloud Model, which offer an alternative mathematical representation for modeling its stochastic processes. The characteristic function enables the implementation of regularization techniques and optimization strategies commonly employed in generative models. By deriving and leveraging the characteristic function, we can optimize the generative process despite the absence of an analytical PDF, thereby enhancing the overall performance of generative models by capturing complex data distributions more effectively.  

The characteristic function is a fundamental concept in probability theory, as it uniquely represents the distribution of real-valued random vectors in a concise manner. Its properties play a crucial role in simplifying theoretical derivations, particularly for complex probability distributions.

Consider the characteristic function of a probability measure $P$ on $\mathcal{R}^d$, given by  

\begin{equation*}
\begin{aligned}
    \Phi_P:\mathcal{R}^d \to \mathcal{C};\quad \boldsymbol{z} \mapsto \Phi_P(\boldsymbol{z}) = \mathbb{E}\left[e^{i \boldsymbol{z}^\top X}\right] = \int e^{i \boldsymbol{z}^\top x} P{\rm d}x,
\end{aligned}
\end{equation*}
where $\mathbb{E}\left[e^{i \boldsymbol{z}^\top X}\right]$ denotes the expectation of the complex exponential of the inner product between $\boldsymbol{z}$ and the random vector $X$, and the integral formulation expresses this expectation with respect to the probability measure $P$.  

Building on the sample generation process outlined in Algorithm~\ref{alg:FCG}, the stochastic generation mechanism of the cloud model can be derived. The generation of cloud droplets ($X$) in the CM follows a two-layer stochastic process:  

\begin{itemize}
\item[1)] Perturbation of entropy ($\text{En}$): $S \sim N(En, He^2)$, where $He$ quantifies the uncertainty associated with $En$.  
\item[2)] Generation of $X$ conditioned on $S$: $X |S \sim N(Ex, S^2)$, where $Ex$ denotes the mathematical expectation of $X$.  
\end{itemize}  

These processes reflect the hierarchical structure of uncertainty representation in the CM.

\begin{theorem}
\label{thm:theorem1}
Given a Cloud Model (CM) defined as $CM = \{Ex: expected value, En: entropy, He: hyper-entropy\}$, its characteristic function can be expressed as:  

\begin{equation} \label{E2}
\begin{aligned}
    \Phi_{X}(\boldsymbol{z}) = \frac{\exp\left(i\boldsymbol{z} \cdot Ex-\frac{\boldsymbol{z}^\top En^2 \boldsymbol{z}}{2(1 + \boldsymbol{z}^\top He^2 \boldsymbol{z})}\right)}{\sqrt{1 + \boldsymbol{z}^\top He^2 \boldsymbol{z}}} 
\end{aligned}
\end{equation}
\end{theorem}

\begin{proof}  
For clarity, the following random variables are presented in scalar form. According to the definition of the characteristic function, we have:  

Given a fixed $S$, $X | S$ follows $N(Ex, S^2)$. Its conditional characteristic function is given by:  

\begin{equation*}  
\begin{aligned}  
    \mathbb{E}_{X|S}[e^{izX}] = \exp\left(iz \cdot Ex - \frac{1}{2} s^2 z^2 \right)  
\end{aligned}  
\end{equation*}  

Substituting the conditional characteristic function into the marginal expectation expression, we obtain:  

\begin{equation*}  
\begin{aligned}  
    \Phi_X(z) &= \mathbb{E}_{S} \left[ \exp\left(iz \cdot Ex - \frac{1}{2} s^2 z^2 \right) \right] \\  
    &= \exp(iz \cdot Ex) \cdot \mathbb{E}_{S} \left[ \exp\left(-\frac{1}{2}s^2 z^2 \right) \right]  
\end{aligned}  
\end{equation*}  

Since $S \sim N(En, He^2)$, the expectation term is computed as:  

\begin{equation*}  
\begin{aligned}  
    \mathbb{E}_{S} \left[ \exp\left(-\frac{1}{2}s^2 z^2 \right) \right] = \frac{\exp\left(-\frac{\text{En}^2 z^2}{2(1 + z^2 \text{He}^2)}\right)}{\sqrt{1 + z^2 \text{He}^2}}  
\end{aligned}  
\end{equation*}  

Substituting this result back into the expression, we obtain the characteristic function of the CM:  

\begin{equation*}  
\begin{aligned}  
    \Phi_X(z) = \frac{\exp\left(iz \cdot Ex-\frac{En^2 z^2}{2(1 + z^2 He^2)}\right)}{\sqrt{1 + z^2 He^2}}  
\end{aligned}  
\end{equation*}  

\end{proof}  

\begin{algorithm}[tb]
   \caption{Cloud Model Characteristic Function Auto-Encoder(CMCFAE)}
   \label{alg:CMCFAE}
\begin{algorithmic}
   \STATE {\bfseries Requires:} Characteristic function \( \Phi_{P_{\boldsymbol{Z}}} \) of the cloud model prior \( P_{\boldsymbol{Z}} \), kernel $k$, regularization coefficient $\lambda > 0$. Initialiaze the generator $\mathcal{G}$ of $\boldsymbol{W}$, encoder $Q_{\theta}$ and decoder $G_{\phi}$.
   \WHILE{$(\theta,\phi)$ not converged}
   \STATE Sample $\boldsymbol{X}$ from the training set;
   \STATE Sample $\boldsymbol{Z}$ from $Q_{\theta}(\boldsymbol{Z}|\boldsymbol{X})$;
   \STATE Sample $\boldsymbol{W}$ from the generator $\mathcal{G}$;
   \STATE Calculate $\mathcal{L}_{\text{CMCFAE}}(\theta,\phi)=\mathcal{L}(\boldsymbol{X},\boldsymbol{Z},\boldsymbol{W},\Phi_{P_{\boldsymbol{Z}}})$;
   \STATE Update $Q_{\theta}$ and $G_{\phi}$ by taking a gradient step towards the minimizer of $\mathcal{L}_{\text{CMCFAE}}(\theta,\phi)$.
   \ENDWHILE
\end{algorithmic}
\end{algorithm}

This characteristic function reflects the probabilistic properties of the cloud model in the characteristic function space, with the following key aspects:

\begin{itemize}
\item \textbf{Central Tendency} The exponential term $\exp(iz \cdot Ex)$ determines the central location of the distribution, controlled by $Ex$, which establishes that the expected value of cloud droplets is $Ex$.  

\item \textbf{Amplitude Decay} The amplitude factor $\frac{1}{\sqrt{1 + z^2 \, He^2}}$ decreases with increasing $z$, indicating the role of hyper-entropy $He$ in controlling the dispersion of the distribution. A larger $He$ results in a wider uncertainty range, leading to a more gradual variation of the characteristic function.  

\item \textbf{Width Adjustment} The term $-\frac{En^2 z^2}{2(1 + z^2 \, He^2)}$ characterizes the influence on the shape of the distribution induced by entropy $\text{En}$ and hyper-entropy $He$. $En$ represents the basic uncertainty range, while $He$ serves as a higher-order modulator, governing the decay rate of the distribution tail.  
\end{itemize}

Based on the preceding analysis, the characteristic function of CM fully encapsulates the central tendency, uncertainty range, and distribution’s dynamic adjustment mechanisms, laying the groundwork for deeper theoretical investigations and practical applications of the cloud model.

\subsection{Cloud Model Characteristic Function Auto-Encoder}

Inspired by prior work on Wasserstein Auto-Encoders employing MMD-based regularization (WAE-MMD), we propose a novel generative model\textemdash \textbf{Cloud Model Characteristic Function Auto-Encoder (CMCFAE)}. The key idea is to integrate the cloud model into the WAE framework and leverage its characteristic function to regularize the latent space, thereby enabling more accurate modeling of complex distributions. This approach enhances the ability of WAE to align with the true data distribution while mitigating the limitations of conventional divergence measures (cf. also Figure~\ref{fig:CMCFAE}).

The characteristic function is incorporated into the computation of the MMD metrics\citep{bruck2024generative}. In contrast to the kernel-based MMD computation\citep{tolstikhin2017wasserstein}, this method offers greater flexibility, enabling data characteristics to be represented by distributions more suitable for the specific context.

\textbf{Wasserstein Auto-Encoder (WAE)} \citep{tolstikhin2017wasserstein} introduces a generative model based on an auto-encoder, which consists of a deterministic decoder $G$ and a potentially stochastic encoder $Q$. The core idea of this model is to minimize the Wasserstein distance $D_{\text{WAE}}(P_X, P_G)$ between the data distribution and the data generated by the decoder of the model. The formula is defined as follows:
\begin{equation*}
\begin{aligned}
    D_{\text{WAE}}(P_X, P_G) &= \inf_{Q(Z|X) \in \mathcal{Q}} \mathbb{E}_{P_X} \mathbb{E}_{Q(Z|X)} \big[d(X, X')\big] 
    + \lambda \cdot D_Z(Q_Z, P_Z),
\end{aligned}
\end{equation*}
where $\mathcal{Q}$ denotes a nonparametric family of probabilistic encoders, $D_Z$ is a general divergence measure between $Q_Z$ and $P_Z$, and $\lambda > 0$ is a hyperparameter. In WAE-MMD, $D_Z$ is computed based on the Maximum Mean Discrepancy (MMD). For a positive-definite reproducing kernel $k : \mathcal{Z} \times \mathcal{Z} \to \mathcal{R}$, the MMD is given by:
\begin{equation*}
\begin{aligned}
    D_Z(Q_Z, P_Z) &= \text{MMD}_k(Q_Z, P_Z) 
    = \left\| \int_{\mathcal{Z}} k(z, \cdot) \, {\rm d}P_Z(z) 
    - \int_{\mathcal{Z}} k(z, \cdot) \, {\rm d}Q_Z(z) \right\|_{\mathcal{H}_k},
\end{aligned}
\end{equation*}
where $\mathcal{H}_k$ denotes the reproducing kernel Hilbert space (RKHS) of real-valued functions mapping $\mathcal{Z}$ to $\mathcal{R}$. If $k$ is characteristic, $\text{MMD}_k$ defines a metric and may serve as a divergence measure.

In WAE-MMD \citep{tolstikhin2017wasserstein}, the computation of MMD depends on the sampling process of the prior distribution $P_Z$. When $k$ is a translation-invariant kernel, \citep{sriperumbudur2010hilbert} introduces a method to compute MMD directly from the characteristics of the prior distribution, bypassing the need for its sampling process. In this approach, $\text{MMD}_k(P_1, P_2)$ can be formulated as: 

\begin{equation*}
\begin{aligned}
    \text{MMD}_k(P_1, P_2) = \left( \mathbb{E} \left\|\Phi_{P_1}(\boldsymbol{W}) - \Phi_{P_2}(\boldsymbol{W})\right\|_2^2 \right)^{1/2}
\end{aligned}
\end{equation*}

Therefore, when $k$ is a translation-invariant kernel, $\text{MMD}_k(P_1, P_2)$ can be interpreted as the expected distance between the characteristic functions $\Phi_{P_1}$ and $\Phi_{P_2}$, evaluated at a random point $\boldsymbol{W}$.

Building on the foundation established by \citep{sriperumbudur2010hilbert, bruck2024generative} proposed a more computationally efficient version of MMD. In the context of the generative model within the WAE framework, MMD measures the distance between the latent space distribution of the encoder $Q_Z$ and the prior distribution $P_Z$. The specific formula is given by:

\begin{equation*}
\begin{aligned}
    &\text{MMD}_k(Q_Z, P_Z) 
    = \mathbb{E}_{\boldsymbol{W}} \left[ \left\| n^{-1} \sum_{i=1}^n \exp(i \boldsymbol{W}^\top \boldsymbol{Z}_i) - \Phi_{P_Z}(\boldsymbol{W}) \right\|_2^2 \right]^{1/2}
\end{aligned}
\end{equation*}

Since $Q_Z$ is generally inaccessible, we must rely on empirical approximations of $Q_Z$. To further simplify the computation, we obtain:
\begin{equation*}
\begin{aligned}
    (\text{MMD}_k(Q_Z, P_Z))^2 = C_Q + C_{QP} + C_P,
\end{aligned}
\end{equation*}
where $C_Q := \frac{1}{n^2} \mathbb{E}_{\boldsymbol{W}} \left[ \sum_{i,j=1}^{n} \exp(i \boldsymbol{W}^\top (\boldsymbol{Z}_i - \boldsymbol{Z}_j)) \right]$ represents the feature embedding difference of the empirical distribution, and $C_{QP} := -\frac{2}{n} \mathbb{E}_{\boldsymbol{W}} \left[ \sum_{i=1}^{n} \exp(-i \boldsymbol{W}^\top \boldsymbol{Z}_i) \Phi_P(\boldsymbol{W}) \right]$ represents the feature interaction term between the empirical and target distributions. Note that $C_P := \mathbb{E}_{\boldsymbol{W}} \left[ \Phi_P(\boldsymbol{W}) \Phi_P(\boldsymbol{W})^\top \right]$ is a constant that depends solely on $P$, representing the feature embedding constant of the target distribution.

It is typically not feasible to assume that $\mathbb{E}_{\boldsymbol{W}} \left[ \sum_{i=1}^{n} \exp(-i \boldsymbol{W}^\top \boldsymbol{Z}_i) \Phi_P(\boldsymbol{W}) \right]$ can be computed in closed form. Therefore, we proceed by approximating $\text{MMD}_k(Q_Z, P_Z)$ using these approximations:

$$
C_{QP} \approx - 2 \Re\left(\frac{1}{n m} \sum_{i=1}^n \sum_{l=1}^m  \exp(-i \boldsymbol{W}_l^\top \boldsymbol{Z}_i) \Phi_P(\boldsymbol{W}_l)\right),
$$
and
$$
C_Q \approx \frac{1}{m n (n - 1)} \sum_{i,j = 1, i \neq j}^n \sum_{l=1}^m \exp(i \boldsymbol{W}_l^\top (\boldsymbol{Z}_i - \boldsymbol{Z}_j)),
$$
where $\Re(z)$ denotes the real part of the complex number $z$.

Therefore, we can derive the optimization term $\Gamma(\boldsymbol{Y}, \boldsymbol{W}, \Phi_P)$, which is approximately equivalent to $\text{MMD}_k(Q_Z, P_Z)$. Specifically, it is defined as:

$$
\Gamma(\boldsymbol{Z}, \boldsymbol{W}, \Phi_P) := C_{Q} + C_{QP},
$$
where $C_{Q}$ and $C_{QP}$ represent the previously mentioned approximate values. Since $\Gamma(\boldsymbol{Z}, \boldsymbol{W}, \Phi_P)$ lacks a constant term $C_P$ compared to $(\text{MMD}_k(Q_Z, P_Z))^2$, this term may become negative during the optimization process.

The loss function of the CMCFAE model, denoted as $\mathcal{L}_{\text{CMCFAE}}$, is defined as:

\begin{equation} \label{E3}
\begin{aligned}
    \mathcal{L}_{\text{CMCFAE}} &:= \underbrace{\inf_{Q(\boldsymbol{Z} \mid \boldsymbol{X}) \in \mathcal{Q}} \mathbb{E}_{P_{\boldsymbol{X}}} \mathbb{E}_{Q(\boldsymbol{Z} \mid \boldsymbol{X})} \left[ d(\boldsymbol{X}, G(\boldsymbol{Z})) \right]}_{\text{data reconstruction error}} 
        + \lambda \cdot \underbrace{\Gamma(\boldsymbol{Z}, \boldsymbol{W}, \Phi_{P_{\boldsymbol{Z}}})}_{\text{MMD metric for} \, \boldsymbol{Z}},
\end{aligned}
\end{equation}
where $\Gamma(\boldsymbol{Z}, \boldsymbol{W}, \Phi_{P_{\boldsymbol{Z}}})$ represents an optimization term that is equivalent to the MMD metric. The prior distribution $P_{\boldsymbol{Z}}$ is modeled using the cloud model, where $\Phi_{P_{\boldsymbol{Z}}}$ denotes the characteristic function of the cloud model, as specified in Equation~\eqref{E2}. In line with standard practices in generative modeling, we employ deep neural networks to parameterize both the encoder $Q$ and the decoder $G$.

Based on the loss function in Equation~\eqref{E3}, we can construct the Cloud Model Characteristic Function Auto-Encoder, as shown in Algorithm~\ref{alg:CMCFAE}.

\begin{table*}[h]
\centering
\caption{\textbf{Performance of CMCFAE compared to other baselines on MNIST, FashionMNIST, CIFAR-10, and CelebA}. In the CMCFAE-FP model, the prior parameters $P_{\boldsymbol{Z}}$ are fixed across different dimensions, specifically $Ex = 0.0$, $En = 1.0$, and $He = 0.1$. In contrast, in the CMCFAE-VP model, the prior parameters $P_{\boldsymbol{Z}}$ vary across different dimensions, specifically $Ex \sim \text{Uniform}[-10.0, 10.0]$, $En \sim \text{Uniform}[1.0, 5.0]$, and $He \sim \text{Uniform}[0.1, 1.0]$. The best results are highlighted in green, and the second-best in light green.}
\label{tal: Quantitative comparisons-table}
\resizebox{\textwidth}{!}{
\begin{tabular}{lcccccccccccc}
\hline
\multirow{2}{*}{\textbf{Model}} & \multicolumn{3}{c}{\textbf{MNIST}} & \multicolumn{3}{c}{\textbf{FashionMNIST}} & \multicolumn{3}{c}{\textbf{CIFAR-10}} & \multicolumn{3}{c}{\textbf{CelebA}}\\ 

\cmidrule(r){2-4} \cmidrule(r){5-7} \cmidrule(r){8-10} \cmidrule(r){11-13}
 & $\lambda$ & \makecell[c]{\textbf{Rec.} \\ \textbf{Error {\textcolor{red}{↓}}}} & \makecell[c]{\textbf{FID} \\ \textbf{Score {\textcolor{red}{↓}}}} & $\lambda$ & \makecell[c]{\textbf{Rec.} \\ \textbf{Error {\textcolor{red}{↓}}}} & \makecell[c]{\textbf{FID} \\ \textbf{Score {\textcolor{red}{↓}}}} & $\lambda$ & \makecell[c]{\textbf{Rec.} \\ \textbf{Error {\textcolor{red}{↓}}}}  & \makecell[c]{\textbf{FID} \\ \textbf{Score {\textcolor{red}{↓}}}} & $\lambda$ & \makecell[c]{\textbf{Rec.} \\ \textbf{Error {\textcolor{red}{↓}}}}  & \makecell[c]{\textbf{FID} \\ \textbf{Score {\textcolor{red}{↓}}}}\\ 
\hline
AE   & $-$     & $11.19$    & $52.74$    & $-$   & $9.87$      & $81.98$    & $-$       & \cellcolor[HTML]{B6D7A8}$\textbf{24.67}$           & $269.09$   & $-$   & $86.41$   & $353.50$   \\       
VAE  & $-$     & $18.79$    & $40.47$    & $-$    & $15.41$   & $64.98$   & $-$          & $63.77$        
 & $172.39$   & $-$   & $110.87$   & $60.85$\\ 
WAE-MMD   & $1.0$     & $11.14$       & $27.65$    & $100.0$    & $10.01$           & $58.79$    & $1.0$         & \cellcolor[HTML]{D9EAD3}$25.04$     & $129.37$   & $100.0$   & $86.38$   & $51.51$\\ 
SWAE  & $1.0 $    & $10.99$    & $29.76$    & $100.0$    & $10.56$    & $54.48$   & $1.0$       & $25.42$     & $141.91$   & $100.0$   & $85.97$   & $53.85$\\ 
CWAE      & $1.0$     & $11.25$        & \cellcolor[HTML]{D9EAD3}$23.63$     & $10.0$   & $10.36$           & \cellcolor[HTML]{D9EAD3}$49.49$    & $1.0$        & $25.93$   & \cellcolor[HTML]{B6D7A8}$\textbf{120.02}$   & $5.0$   & $86.89$   & $49.69$\\ 
CMCFAE-FP      & $10.0$      & \cellcolor[HTML]{D9EAD3}$9.23$       & \cellcolor[HTML]{B6D7A8}$\textbf{22.03}$    & $10.0$    & \cellcolor[HTML]{B6D7A8}$\textbf{8.98}$           & \cellcolor[HTML]{B6D7A8}$\textbf{35.54}$   & $10.0$        & $25.38$     & \cellcolor[HTML]{D9EAD3}$123.23$   & $100.0$   & \cellcolor[HTML]{D9EAD3}$72.23$   & \cellcolor[HTML]{D9EAD3}$44.95$\\ 
CMCFAE-VP      & $10.0$      & \cellcolor[HTML]{B6D7A8}$\textbf{9.14}$       & $24.45$    & $10.0$    & \cellcolor[HTML]{D9EAD3}$9.26$           & $51.15$   & $10.0$        & $26.33$     & $127.48$   & $100.0$   & \cellcolor[HTML]{B6D7A8}$\textbf{69.81}$   & \cellcolor[HTML]{B6D7A8}$\textbf{43.87}$\\ 
\hline
\end{tabular}
}
\end{table*}

\begin{figure*}[!htb]
\begin{center}
\centerline{\includegraphics[width=\textwidth]{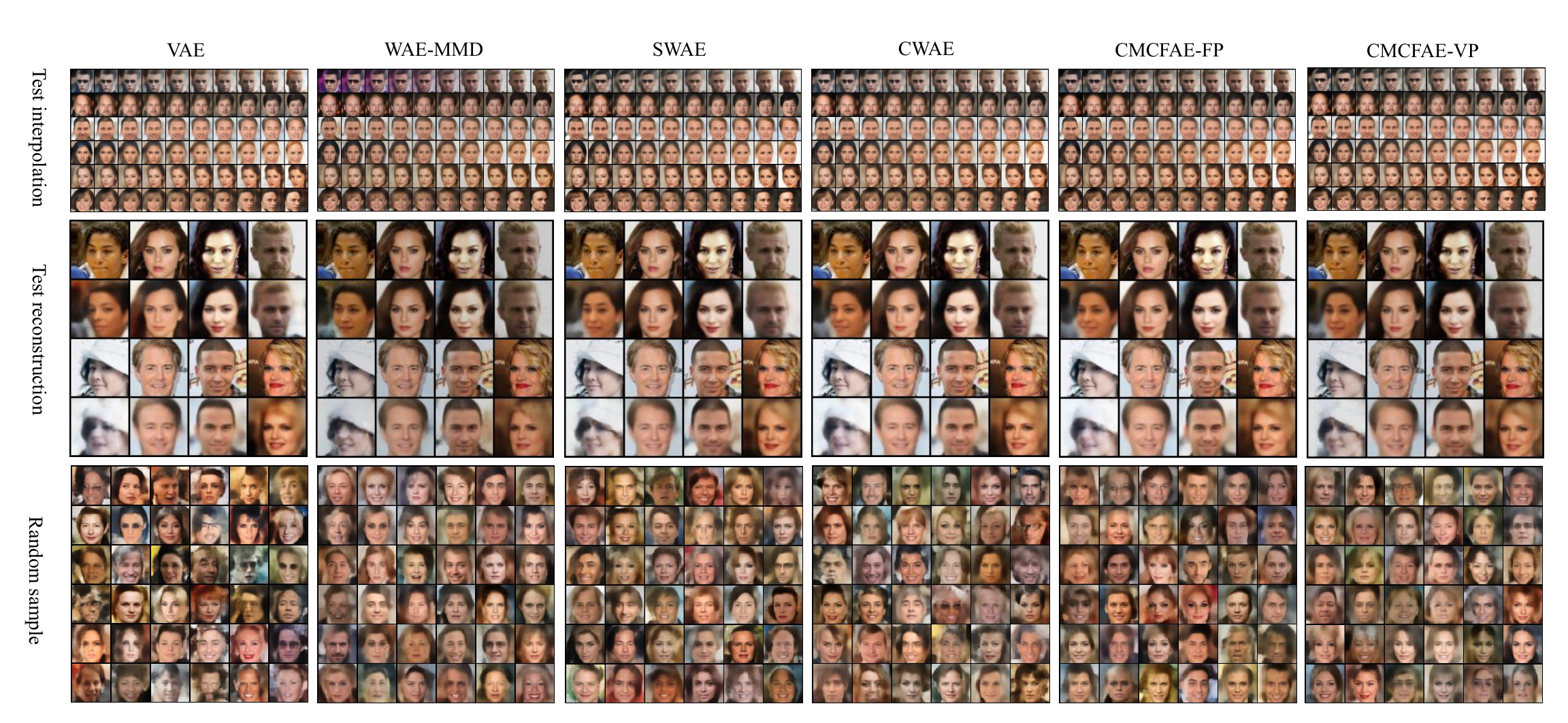}}
\caption{The test results for the VAE, WAE-MMD, SWAE, CWAE, and CMCFAE models on the CelebA dataset include Test Interpolation, Test Reconstruction, and Random Samples. Notably, the odd-numbered rows in the Test Reconstruction results correspond to the respective ground truth input data.}
\label{celeba-qualitative}
\end{center}
\end{figure*}

\section{Experiments} \label{sec4:Experiments}
In this section, we empirically evaluate the proposed CMCFAE model. We compare the proposed CMCFAE model with WAE-MMD\citep{tolstikhin2017wasserstein}, SWAE\citep{2018Sliced}, and CWAE\citep{knop2020cramer} on standard benchmarks, including the MNIST, FashionMNIST, CIFAR-10, and CelebA datasets.

\textbf{Experiment Setup}.  
In the experiments, we utilized the characteristic function form of the cloud model prior distribution $P_{\boldsymbol{Z}}(\boldsymbol{Z})$, expressed as $\Phi_{P_{\boldsymbol{Z}}}(\boldsymbol{Z}; \boldsymbol{Ex}, \boldsymbol{En}, \boldsymbol{He})$ in the latent space $\mathcal{Z}$, where $\boldsymbol{Ex} = \boldsymbol{0}$, and $\boldsymbol{En}$ and $\boldsymbol{He}$ vary across different datasets. For data points, we use the squared error $d(y, \hat{y}) = \lVert y - \hat{y} \rVert_2^2$. We employ a convolutional deep neural network architecture to implement the encoder mapping $Q_{\theta}: \mathcal{X} \rightarrow \mathcal{Z}$ and the decoder mapping $G_{\phi}: \mathcal{Z} \rightarrow \mathcal{X}$. In the different experiments, we tested various values of $\lambda$, specifically $\lambda = 1.0$, $\lambda = 10.0$, and $\lambda = 100.0$.

\subsection{Quantitative tests}
To quantitatively compare CMCFAE with other models, we adopted the experimental setup and neural network architecture described in \citep{knop2020cramer}. We use reconstruction error and the Fréchet Inception Distance (FID) \citep{heusel2017gans} as evaluation metrics.

We observed that, except for the CIFAR-10 dataset where our model did not achieve the best performance, it consistently outperformed all other models on the remaining datasets. Additionally, the prior $P_{\boldsymbol{Z}}$ was found to significantly impact the experimental outcomes, with its influence varying across different datasets. Detailed experimental results are shown in Table~\ref{tal: Quantitative comparisons-table}.

\subsection{Qualitative tests}
The quality of generative models can be evaluated by examining the generated samples, interpolation between samples in the latent space, and random sampling from the reconstructed samples. In Figure~\ref{celeba-qualitative}, we present a comparison of CMCFAE with other methods, using the same network architecture as WAE-MMD and CWAE. The first row shows interpolation results between two random samples from the test set. The second row tests the reconstruction of random samples from the test set. The third row demonstrates the reconstruction of samples using random values drawn from the prior distribution as latent variables. The experiments indicate no perceptible difference between CMCFAE, WAE-MMD, CWAE, and SWAE.

In the next experiment, we conducted a focused evaluation of the latent space distribution under different priors. The comparison method, WAE, used a standard normal distribution as the prior, while we used a cloud model as the prior. We used the t-SNE\citep{tu2018unified} dimensionality reduction algorithm to map the latent space to a two-dimensional space for visualization, with data labels distinguished by different colors. The experiment was conducted on the MNIST dataset, and the specific results are shown in Figure~\ref{latent_space}. In the latent space of WAE, which uses a standard normal distribution as the prior, the boundaries between data points representing the digits 4, 7, and 9 in the MNIST dataset are unclear, and the overlapping of data points is more pronounced compared to Figure~\ref{latent_space}(b). This experiment further validates that the cloud model, as a generalized normal distribution for describing the latent space, more accurately reflects the true conditions in complex networks compared to the standard normal distribution\citep{li2004study}.

\begin{figure*}[!htb]
\begin{center}
\centerline{\includegraphics[width=12cm]{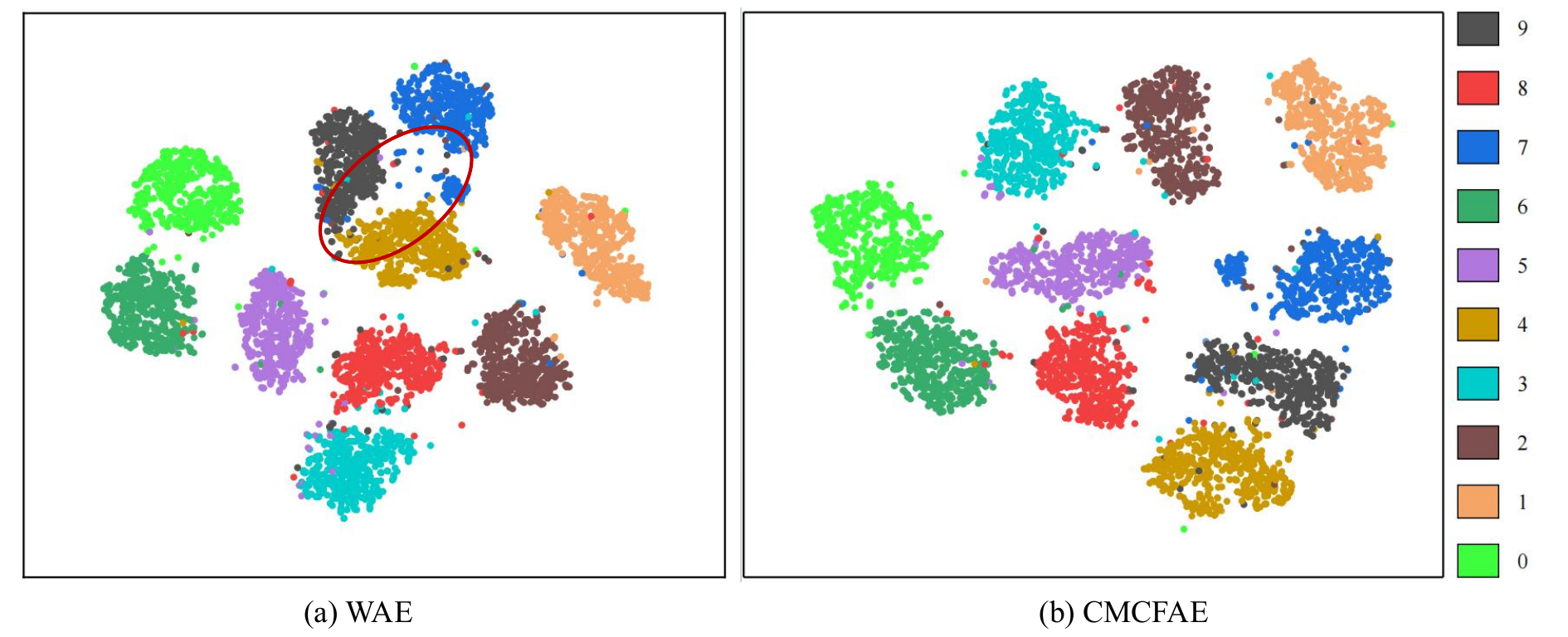}}
\caption{The latent space distributions of WAE and CMCFAE based on the MNIST dataset are compared. Both models utilize the same network architecture. WAE uses a standard normal distribution $N(0, I)$ as the prior, while CMCFAE employs a cloud model (which can be understood as a generalized normal distribution with three numerical characteristics: $Ex = 0$, $En = 1$, and $He = 0.1$) as the prior.}
\label{latent_space}
\end{center}
\end{figure*}

\section{Conclusion}
In this paper, we introduced the Cloud Model Characteristic Function Auto-Encoder (CMCFAE), a novel generative framework that integrates the cloud model with Maximum Mean Discrepancy (MMD). By deriving the characteristic function of the cloud model, thereby circumventing the limitations imposed by the lack of an analytical probability density function, we provide a rigorous method for regularizing the latent space in generative models. Our formulation, which includes a well-defined loss function and an efficient algorithmic process, enables CMCFAE to capture complex latent distributions more accurately than conventional approaches. Extensive quantitative and qualitative evaluations on benchmark datasets, including MNIST, FashionMNIST, CIFAR-10, and CelebA, demonstrate that CMCFAE achieves superior results in reconstruction quality, latent space structuring, and sample diversity. This work represents the first integration of cloud model theory with MMD through the derivation of its characteristic function, thereby opening new avenues for enhancing autoencoder-based generative models and expanding the applicability of cloud model theory in real-world scenarios. We anticipate that the proposed approach will inspire further advancements in generative modeling and facilitate more robust representation learning.

\begin{appendices}

\appendix
\section{Proofs for Section~\ref{sec:Methodology}}
\subsection{Proof of Theorem~\ref{thm:theorem1}} 

The calculation process for $\mathbb{E}_{S} \left[ \exp\left(-\frac{1}{2}s^2 z^2 \right) \right]$ is provided below:

\begin{equation*}
\begin{aligned}
    \mathbb{E}_S[\exp(-s^2 z^2 / 2)] = \int_{-\infty}^{+\infty} \exp\left(-\frac{s^2 z^2}{2}\right) f(s) {\rm d}s 
\end{aligned}
\end{equation*}
where \( f(s) \) is the probability density function of \( s \). Given \( S \sim N(En, He^2) \), its probability density function is:
\begin{equation*}
\begin{aligned}
    f(s) = \frac{1}{\sqrt{2\pi He^2}} \exp\left(-\frac{(s - En)^2}{2 \text{He}^2}\right)
\end{aligned}
\end{equation*}

Substituting \( f(s) \), we obtain:
\begin{equation*}
\begin{aligned}
    &\mathbb{E}_S[\exp(-s^2 z^2 / 2)]\\ &= \int_{-\infty}^{+\infty} \exp\left(-\frac{s^2 z^2}{2}\right) \cdot \frac{1}{\sqrt{2\pi He^2}} \exp\left(-\frac{(s - En)^2}{2 He^2}\right) {\rm d}s \\
    &= \frac{1}{\sqrt{2\pi He^2}} \int_{-\infty}^{+\infty} \exp\left(-\frac{s^2 z^2}{2} - \frac{(s - En)^2}{2 He^2}\right) {\rm d}s \\
    &= \frac{1}{\sqrt{2\pi He^2}} \exp\left(-\frac{En^2 z^2}{2(1 + z^2 He^2)}\right) \int_{-\infty}^{+\infty} \exp\left(-\frac{1}{2} \left(z^2 + \frac{1}{He^2}\right) \left(s - \frac{En / He^2}{z^2 + 1 / He^2}\right)^2\right) {\rm d}s
\end{aligned}
\end{equation*}

The integral part is calculated as:
\begin{equation*}
\begin{aligned}
    \int_{-\infty}^{+\infty} \exp\left(-\frac{1}{2} \left(z^2 + \frac{1}{He^2}\right) \left(s - \frac{En / He^2}{z^2 + 1 / He^2}\right)^2\right) {\rm d}s = \sqrt{\frac{2 \pi}{z^2 + \frac{1}{\text{He}^2}}}
\end{aligned}
\end{equation*}

After simplification, the final result is obtained as:
\begin{equation*}
\begin{aligned}
    \mathbb{E}_S[\exp(-s^2 z^2 / 2)] = \frac{1}{\sqrt{1 + z^2 He^2}} \exp\left(-\frac{En^2 z^2}{2(1 + z^2 He^2)}\right)
\end{aligned}
\end{equation*}

\section{Experimental Details}
I now report the key details of all experiments. All our experiments are built upon the open-source baseline codebase\citep{tolstikhin2017wasserstein}.

\subsection{Datasets} 
We employed the following datasets to evaluate the performance of CMCFAE and other methods from both quantitative and qualitative perspectives:

\textbf{MNIST}\citep{simard2003best} is a benchmark dataset for handwritten digit recognition, consisting of 70,000 grayscale images of digits (0–9). It is divided into 60,000 training images and 10,000 test images, with each image having a resolution of 28×28 pixels. The labels are provided as integers corresponding to the digit in each image. Due to its simplicity and broad applicability, MNIST serves as a standard dataset for evaluating machine learning models in classification tasks, including those based on generative models and autoencoders.

\textbf{FashionMNIST}\citep{xiao2017fashion} serves as a more challenging alternative to MNIST, designed for benchmarking machine learning models in image classification tasks. It consists of 70,000 grayscale images of fashion items across 10 categories, such as T-shirts, trousers, and bags, all at a resolution of 28×28 pixels. The dataset is split into 60,000 training images and 10,000 test images. Unlike MNIST’s handwritten digits, FashionMNIST represents real-world object categories, making it suitable for testing models' ability to generalize to more complex and varied data distributions.

\textbf{CIFAR-10}\citep{krizhevsky2009learning} is a widely used benchmark for image classification and generative modeling tasks. It consists of 60,000 color images, each with a resolution of 32×32 pixels, and evenly distributed across 10 classes, including airplanes, automobiles, birds, cats, and dogs. The dataset is divided into 50,000 training images and 10,000 test images. Each image contains a single object, centered and labeled with its corresponding class. CIFAR-10 is known for its moderate complexity and is frequently employed to evaluate the performance of deep learning models, particularly those involving convolutional architectures.

\textbf{CelebA}\citep{liu2015deep} is a large-scale facial attributes dataset containing more than 200,000 images of celebrity faces. Each image is annotated with 40 binary attributes, such as gender, age, and hairstyle, as well as five landmark points for alignment. CelebA is widely used in tasks such as facial attribute classification, face detection, and generative modeling. The dataset's high variability in pose, lighting, and expression provides a robust benchmark for evaluating models' ability to handle diverse, real-world face data distributions.

\subsection{Baselines} 
In our experiments, we evaluated the quantitative and qualitative performance of CMCFAE on image generation tasks by comparing it with several selected baselines, ensuring both a fair comparison and comprehensive coverage of different methods reported in the literature. Specifically, for the quantitative evaluation, we selected AE and VAE\citep{kingma2013auto} as baselines, along with WAE\citep{tolstikhin2017wasserstein} and its optimized variants SWAE\citep{2018Sliced} and CWAE\citep{knop2020cramer}, which are based on the WAE framework. For the qualitative evaluation, the AE model was excluded as a baseline.

\subsection{Hyperparameter Tuning Strategy} 
As detailed in Section~\ref{sec4:Experiments}, the value of $\lambda$ in the loss function varies across datasets. Specifically, $\lambda$ is set to 10 for MNIST, FashionMNIST, and CIFAR-10, whereas it is set to 100 for CelebA. In the experiments, the primary distinction between the CMCFAE-FP and CMCFAE-VP models lies in the priors they employ. Specifically, CMCFAE-FP uses a fixed cloud model prior with parameters $Ex=0$, $En=1$, and $He=0.1$, while CMCFAE-VP utilizes a prior where $Ex \sim \text{Uniform}[-10.0, 10.0]$, $En \sim \text{Uniform}[1.0, 5.0]$, and $He \sim \text{Uniform}[0.1, 1.0]$.

\subsection{Training Details} 
For \textbf{MNIST} and \textbf{FashionMNIST}, we used a batch size of 100 and trained the model for 200 epochs. The encoder-decoder pair was optimized using the Adam optimizer, initialized with a learning rate of $\alpha = 10^{-3}$, a first-order momentum of $\beta_1 = 0.5$, and a second-order momentum of $\beta_2 = 0.999$.

Both the encoder and decoder utilized fully convolutional architectures with 4x4 convolutional filters. 

Encoder Architecture:
\begin{equation*}
\begin{aligned}
    x \in \mathcal{R}^{28 \times 28} &\rightarrow \text{Conv}_{128} \rightarrow \text{BN} \rightarrow \text{ReLU} \\ 
    &\rightarrow \text{Conv}_{256} \rightarrow \text{BN} \rightarrow \text{ReLU} \\
    &\rightarrow \text{Conv}_{512} \rightarrow \text{BN} \rightarrow \text{ReLU} \\
    &\rightarrow \text{Conv}_{1024} \rightarrow \text{BN} \rightarrow \text{ReLU} \\
    &\rightarrow \text{FC}_8
\end{aligned}
\end{equation*}

Decoder Architecture:
\begin{equation*}
\begin{aligned}
    z \in \mathcal{R}^8 &\rightarrow \text{FC}_{7 \times 7 \times 1024} \\
    &\rightarrow \text{FSConv}_{512} \rightarrow \text{BN} \rightarrow \text{ReLU} \\
    &\rightarrow \text{FSConv}_{256} \rightarrow \text{BN} \rightarrow \text{ReLU} \\
    &\rightarrow \text{FSConv}_1 \rightarrow \text{Sigmoid}
\end{aligned}
\end{equation*}

For \textbf{CIFAR-10}, we employed a mini-batch size of 100 and trained the model for 300 epochs. The learning rate was initialized at $\alpha = 10^{-3}$ for both the encoder and decoder pairs, with the Adam optimizer's first-order momentum set to $\beta_1 = 0.5$ and second-order momentum set to $\beta_2 = 0.999$.

The encoder utilizes a fully convolutional architecture with 2x2 convolutional filters. In the decoder, three fractional-strided convolutions utilize 3x3 convolutional filters in a fully convolutional architecture, while one transposed convolution layer employs a 4x4 convolutional filter in a fully convolutional setup.

Encoder Architecture:
\begin{equation*}
\begin{aligned}
    x \in \mathcal{R}^{32 \times 32 \times 3} &\rightarrow \text{Conv}_{32} \rightarrow \text{ReLU} \\ 
    &\rightarrow \text{Conv}_{32} \rightarrow \text{ReLU} \\
    &\rightarrow \text{Conv}_{32} \rightarrow \text{ReLU} \\
    &\rightarrow \text{Conv}_{32} \rightarrow \text{ReLU} \\
    &\rightarrow \text{FC}_{128} \rightarrow \text{ReLU} \\
    &\rightarrow \text{FC}_{64}
\end{aligned}
\end{equation*}

Decoder Architecture:
\begin{equation*}
\begin{aligned}
    z \in \mathcal{R}^{64} &\rightarrow \text{FC}_{128} \rightarrow \text{ReLU} \\
    &\rightarrow \text{FC}_{32 \times 16 \times 16} \rightarrow \text{ReLU}\\
    &\rightarrow \text{FSConv}_{32} \rightarrow \text{ReLU} \\
    &\rightarrow \text{FSConv}_{32} \rightarrow \text{ReLU} \\
    &\rightarrow \text{FSConv}_{32} \rightarrow \text{ReLU} \\
    &\rightarrow \text{FSConv}_{32} \rightarrow \text{ReLU} \rightarrow \text{Sigmoid}
\end{aligned}
\end{equation*}

For \textbf{CelebA}, we employed a mini-batch size of 100 and trained the model for 250 epochs. The learning rate was initialized at $\alpha = 10^{-3}$, with the Adam optimizer's first-order momentum set to $\beta_1 = 0.5$ and second-order momentum set to $\beta_2 = 0.999$.

Both the encoder and decoder utilized fully convolutional architectures with 4x4 convolutional filters.

Encoder Architecture:
\begin{equation*}
\begin{aligned}
    x \in \mathcal{R}^{64 \times 64 \times 3} &\rightarrow \text{Conv}_{128} \rightarrow \text{BN} \rightarrow \text{ReLU} \\ 
    &\rightarrow \text{Conv}_{256} \rightarrow \text{BN} \rightarrow \text{ReLU} \\
    &\rightarrow \text{Conv}_{512} \rightarrow \text{BN} \rightarrow \text{ReLU} \\
    &\rightarrow \text{Conv}_{1024} \rightarrow \text{BN} \rightarrow \text{ReLU} \\
    &\rightarrow \text{FC}_{64}
\end{aligned}
\end{equation*}

Decoder Architecture:
\begin{equation*}
\begin{aligned}
    z \in \mathcal{R}^{64} &\rightarrow \text{FC}_{8 \times 8 \times 1024} \\
    &\rightarrow \text{FSConv}_{512} \rightarrow \text{BN} \rightarrow \text{ReLU} \\
    &\rightarrow \text{FSConv}_{256} \rightarrow \text{BN} \rightarrow \text{ReLU} \\
    &\rightarrow \text{FSConv}_{128} \rightarrow \text{BN} \rightarrow \text{ReLU} \\
    &\rightarrow \text{FSConv}_3
\end{aligned}
\end{equation*}

Here, $\text{Conv}_k$ denotes a convolution with $k$ filters, $\text{FSConv}_k$ represents a fractional strided convolution using $k$ filters, BN refers to batch normalization, ReLU stands for the Rectified Linear Unit, Sigmoid represents the logistic sigmoid function, and $\text{FC}_k$ denotes a fully connected mapping to $\mathcal{R}^k$.

\section{Additional Results}
The training dynamics of MNIST, FashionMNIST, and CIFAR-10, as shown in Figure~\ref{loss_all}.

\begin{figure*}[ht]
\begin{center}
\centerline{\includegraphics[width=\textwidth]{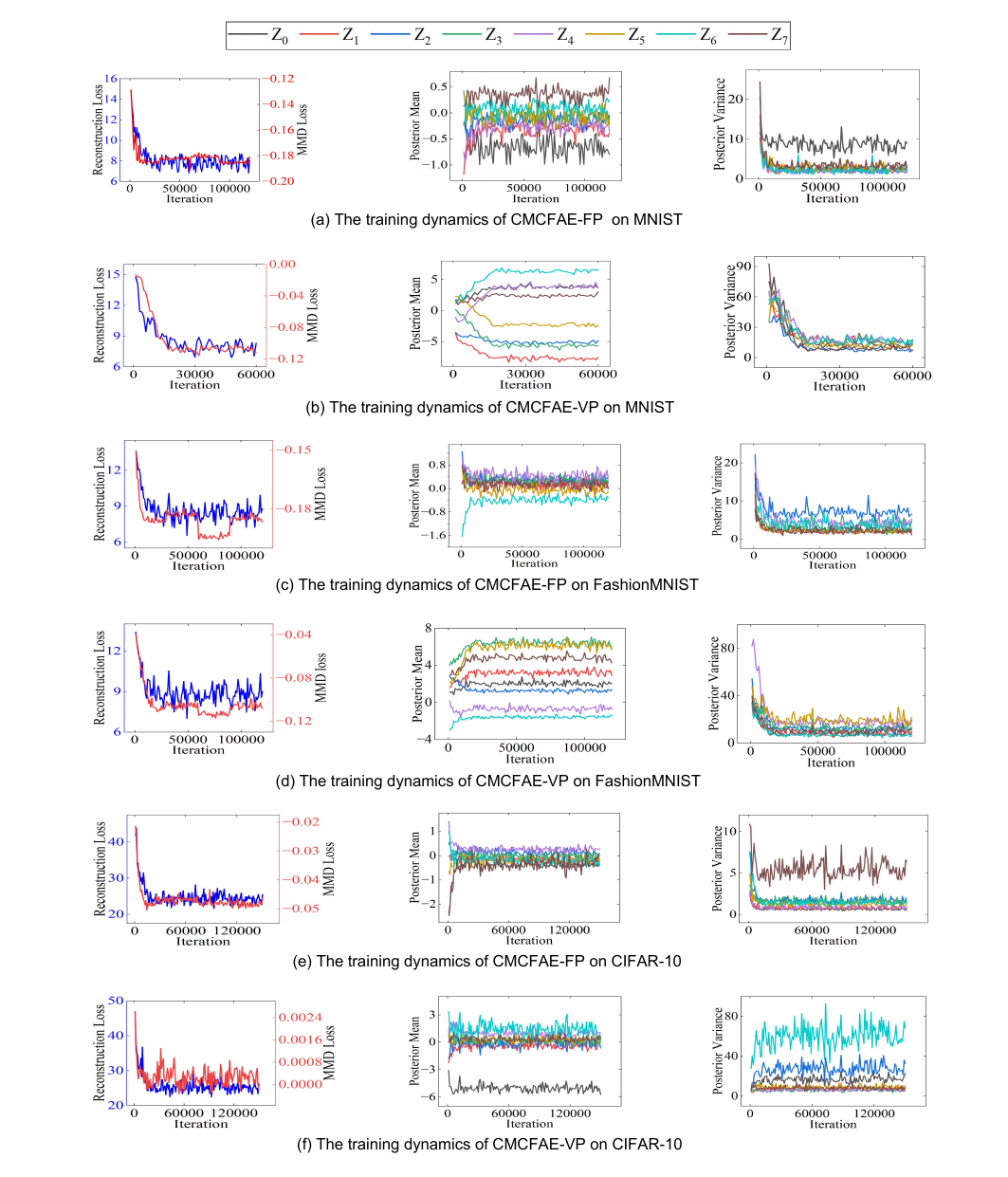}}
\caption{The training dynamics of CMCFAE-FP and CMCFAE-VP on MNIST, FashionMNIST, and CIFAR-10, focusing on reconstruction loss, MMD loss, and the variations in the posterior distributions. Specifically, the first column illustrates the evolution of reconstruction loss and MMD loss, while the second and third columns show the progression of the mean and variance of the posterior distributions, respectively. It is important to note that a negative MMD loss during training is a typical occurrence, as discussed in Section~\ref{sec:Methodology}, where the constant term $C_P$ is excluded from the MMD loss calculation. For CIFAR-10, the dimensionality of the latent variables is 64. For clarity, only the progression of the mean and variance of the first 8 dimensions of the latent variables during training is presented.}
\label{loss_all}
\end{center}
\end{figure*}




\end{appendices}

% \bibliographystyle{plainnat}
\bibliography{sn-bibliography}% common bib file

\end{document}